\documentclass{llncs}
\usepackage{amsmath}
\usepackage{amssymb}
\usepackage{times}
\usepackage{indentfirst}
\usepackage{graphicx}

\begin{document}

\title{Matroidal structure of rough sets based on serial and transitive relations}

\author{Yanfang Liu, William Zhu~\thanks{Corresponding author.
E-mail: williamfengzhu@gmail.com (William Zhu)}}
\institute{
Lab of Granular Computing\\
Zhangzhou Normal University, Zhangzhou 363000, China}



\date{\today}          
\maketitle

\begin{abstract}
The theory of rough sets is concerned with the lower and upper approximations of objects through a binary relation on a universe.
It has been applied to machine learning, knowledge discovery and data mining.
The theory of matroids is a generalization of linear independence in vector spaces.
It has been used in combinatorial optimization and algorithm design.
In order to take advantages of both rough sets and matroids, in this paper we propose a matroidal structure of rough sets based on a serial and transitive relation on a universe.
We define the family of all minimal neighborhoods of a relation on a universe, and prove it satisfy the circuit axioms of matroids when the relation is serial and transitive.
In order to further study this matroidal structure, we investigate the inverse of this construction: inducing a relation by a matroid.
The relationships between the upper approximation operators of rough sets based on relations and the closure operators of matroids in the above two constructions are studied.
Moreover, we investigate the connections between the above two constructions.

\end{abstract}


\section{Introduction}
The theory of rough sets~\cite{Pawlak82Rough} proposed by Pawlak is an extension of set theory for handling incomplete and inexact knowledge in information and decision systems.
And it has been successfully applied to many fields, such as machine learning, granular computing~\cite{CalegariCiucci10}, data mining, approximate reasoning, attribute reduction~\cite{HallHolmes03Benchmarking,MiaoZhaoYaoLiXu09relativereducts,MinHeQianZhu11Test,ZhangQiuWu07ageneralapproach}, rule induction and others~\cite{ChenMiaoWang10aroughsetapproach,DaiXu12approximations}.
For an equivalence relation on a unverse, a rough set is a formal approximation of a crisp set in terms of a pair of sets which give the lower and the upper approximations of the original set.
In order to meet many real applications, the rough sets have been extended to generalized rough sets based on relations~\cite{Kryszkiewicz98Rough,LiuWang06Semantic,QinYangPei08generalizedroughsets,SlowinskiVanderpooten00AGeneralized,Yao98Constructive,Yao98Relational}, and covering-based rough sets~\cite{BonikowskiBryniarskiWybraniecSkardowska98Extensions,Zhu07Topological,ZhuWang03Reduction}.
In this paper, we focus on generalized rough sets based on relations.

Matroid theory~\cite{Lai01Matroid} proposed by Whitney is a generalization of linear algebra and graph theory.
And matroids have been used in diverse fields, such as combinatorial optimization, algorithm design,  information coding and cryptology.
Since a matroid can be defined by many different but equivalent ways, matroid theory has powerful axiomatic systems.
Matroids have been connected with other theories, such as rough sets~\cite{LiuZhu12characteristicofpartition-circuitmatroid,LiuZhuZhang12Relationshipbetween,TangSheZhu12matroidal,WangZhuZhuMin12matroidalstructure}, generalized rough sets based on relations~\cite{ZhuWang11Matroidal,ZhangWangFengFeng11reductionofrough}, covering-based rough sets~\cite{WangZhu11Matroidal,WangZhuZhuMin12quantitative} and lattices~\cite{AignerDowling71matchingtheory,Mao06TheRelation,Matus91abstractfunctional}.

Rough sets and matroids have their own application fields in the real world.
In order to make use of both rough sets and matroids, researchers have combined them with each other and connected them with other theories.
In this paper, we propose a matroidal structure based on a serial and transitive relation on a universe and study its relationships with the lower and upper approximations of generalized rough sets based on relations.

First, we define a family of sets of a relation on a universe, and we call it the family of all minimal neighborhoods of the relation.
When the relation is serial and transitive, the family of all minimal neighborhoods satisfies the circuit axioms of matroids, then a matroid is induced.
And we say the matroid is induced by the serial and transitive relation.
Moreover, we study the independent sets of the matroid through the neighborhood, the lower and upper approximation operators of generalized rough sets based on relations, respectively.
A sufficient and necessary condition, when different relations generate the same matroid, is investigated.
And the relationships between the upper approximation operator of a relation and the closure operator of the matroid induced by the relation are studied.
We employ a special type of matroids, called 2-circuit matroids, which is introduced in~\cite{WangZhuZhuMin12matroidalstructure}.
When a relation is an equivalence relation, the upper approximation operator of the relation is equal to the closure operator of the matroid induced by the relation if and only if the matroid is a 2-circuit matroid.
In order to study the matroidal structure of the rough set based on a serial and transitive relation, we investigate the inverse of the above construction.
In other words, we construct a relation from a matroid.
Through the connectedness in a matroid, a relation can be obtained and proved to be an equivalence relation.
A sufficient and necessary condition, when different matroids induce the same relation, is studied.
Moreover, the relationships between the closure operator of a matroid and the upper approximation operator of the relation induced by the matroid are investigated.
Especially, for a matroid on a universe, its closure operator is equal to the upper approximation operator of the induced equivalence relation if and only if the matroid is a 2-circuit matroid.

Second, the relationships between the above two constructions are studied.
On one hand, for a matroid on a universe, it can induce an equivalence relation, and the equivalence relation can generate a matroid, we prove that the circuit family of the original matroid is finer than one of the induced matroid.
And the original matroid is equal to the induced matroid if and only if the circuit family of the original matroid is a partition.
On the other hand, for a reflexive and transitive relation on a universe, it can generate a matroid, and the matroid can induce an equivalence relation, then the relationship between the equivalence relation and the original relation is studied.
The original relation is equal to the induced equivalence relation if and only if the original relation is an equivalence relation.

The rest of this paper is organized as follows: In Section~\ref{S:preliminaries}, we recall some basic definitions of generalized rough sets based on relations and matroids.
In Section~\ref{S:twoinductionsbetweenmatroidandrelation}, we propose two constructions between relations and matroids.
For a relation on a universe, Subsection~\ref{S:inductionofmatroidbyrelation} defines a family of sets called the family of all minimal neighborhoods and proves it to satisfy the circuit axioms of matroids when the relation is serial and transitive.
The independent sets of the matroid are studied.
And the relationships between the upper approximation operator of a relation and the closure operator of the matroid induced by the relation are investigated.
In Subsection~\ref{S:inductionofrelationbymatroid}, through the circuits of a matroid, we construct a relation and prove it an equivalence relation.
The relationships between the closure operator of the matroid and the upper approximation operator of the induced equivalence relation are studied.
Section~\ref{S:relationshipsbetweenthetwoinductions} represents the relationships between the two constructions proposed in Section~\ref{S:twoinductionsbetweenmatroidandrelation}.
Finally, we conclude this paper in Section~\ref{S:conclusions}.

\section{Preliminaries}
\label{S:preliminaries}
In this section, we recall some basic definitions and related results of generalized rough sets and matroids which will be used in this paper.

\subsection{Relation based generalized rough sets}
Given a universe and a relation on the universe, they form a rough set.
In this subsection, we introduce some concepts and properties of generalized rough sets based on relations~\cite{Yao98Constructive}.
The neighborhood is important and used to construct the approximation operators.

\begin{definition}(Neighborhood~\cite{Yao98Constructive})
\label{D:neighborhood}
Let $R$ be a relation on $U$.
For any $x\in U$, we call the set $RN_{R}(x)=\{y\in U:xRy\}$ the successor neighborhood of $x$ in $R$.
When there is no confusion, we omit the subscript $R$.
\end{definition}

In the following definition, we introduce the lower and upper approximation operators of generalized rough sets based on relations through the neighborhood.

\begin{definition}(Lower and upper approximation operators~\cite{Yao98Constructive})
\label{D:approximationoperators}
Let $R$ be a relation on $U$.
A pair of operators $L_{R}, H_{R}:2^{U}\rightarrow 2^{U}$ are defined as follows: for all $X\subseteq U$,
\centerline{$L_{R}(X)=\{x\in U:RN(x)\subseteq X\}$,}
\centerline{~~~~~~$H_{R}(X)=\{x\in U:RN(x)\cap X\neq\emptyset\}$.}
$L_{R}, H_{R}$ are called the lower and upper approximation operators of $R$, respectively.
We omit the subscript $R$ when there is no confusion.
\end{definition}

We present properties of the lower and upper approximation operators in the following proposition.

\begin{proposition}(\cite{Yao98Constructive})
\label{P:propertiesofapproximations}
Let $R$ be a relation on $U$ and $X^{c}$ the complement of $X$ in $U$.
For all $X, Y\subseteq U$,\\
(1L) $L(U)=U$;\\
(1H) $H(\emptyset)=\emptyset$;\\
(2L) $X\subseteq Y\Rightarrow L(X)\subseteq L(Y)$;\\
(2H) $X\subseteq Y\Rightarrow H(X)\subseteq H(Y)$;\\
(3L) $L(X\cap Y)=L(X)\cap L(Y)$;\\
(3H) $H(X\cup Y)=H(X)\cup H(Y)$;\\
(4LH) $L(X^{c})=(H(X))^{c}$.
\end{proposition}

In~\cite{Zhu07Generalized}, Zhu has studied generalized rough sets based on relations and investigated conditions for a relation to satisfy some of common properties of classical lower and upper approximation operators.
In the following proposition, we introduce one result used in this paper.

\begin{proposition}
\label{P:reflexiverelation}
Let $R$ be a relation on $U$.
Then,\\
\centerline{$R$ is reflexive $\Leftrightarrow$ $X\subseteq H(X)$ for all $X\subseteq U$.}
\end{proposition}

In the following definition, we use the neighborhood to describe a serial relation and a transitive relation on a universe.

\begin{definition}(\cite{Yao98Constructive})
\label{D:serialandtransitive}
Let $R$ be a relation on $U$.\\
(1) $R$ is serial $\Leftrightarrow\forall x\in U[RN(x)\neq\emptyset]$;\\
(2) $R$ is transitive $\Leftrightarrow\forall x, y\in U[y\in RN(x)\Rightarrow RN(y)\subseteq RN(x)]$.
\end{definition}

\subsection{Matroids}
Matroids have many equivalent definitions.
In the following definition, we will introduce one that focuses on independent sets.

\begin{definition}(Matroid~\cite{Lai01Matroid})
\label{D:matroid}
A matroid is a pair $M=(U, \mathbf{I})$ consisting of a finite universe $U$ and a collection $\mathbf{I}$ of subsets of $U$ called independent sets satisfying the following three properties:\\
(I1) $\emptyset\in\mathbf{I}$;\\
(I2) If $I\in \mathbf{I}$ and $I'\subseteq I$, then $I'\in \mathbf{I}$;\\
(I3) If $I_{1}, I_{2}\in \mathbf{I}$ and $|I_{1}|<|I_{2}|$, then there exists $u\in I_{2}-I_{1}$ such that $I_{1}\cup \{u\}\in \mathbf{I}$, where $|I|$ denotes the cardinality of $I$.
\end{definition}

Since the above definition is from the viewpoint of independent sets to represent matroids, it is also called the independent set axioms of matroids.
In order to make some expressions brief, we introduce several symbols as follows.

\begin{definition}(\cite{Lai01Matroid})
\label{D:symbol}
Let $U$ be a finite universe and $\mathbf{A}$ a family of subsets of $U$.
Several symbols are defined as follows:\\
$FMIN(\mathbf{A})=\{X\in\mathbf{A}:\forall Y\in\mathbf{A}, Y\subseteq X\Rightarrow X=Y\}$;\\
$Upp(\mathbf{A})=\{X\subseteq U:\exists A\in\mathbf{A}$ s.t. $A\subseteq X\}$;\\
$Opp(\mathbf{A})=\{X\subseteq U:X\notin\mathbf{A}\}$.
\end{definition}

In a matroid, a subset is a dependent set if it is not an independent set.
Any circuit of a matroid is a minimal dependent set.

\begin{definition}(Circuit~\cite{Lai01Matroid})
\label{D:circuit}
Let $M=(U, \mathbf{I})$ be a matroid.
Any minimal dependent set in $M$ is called a circuit of $M$, and we denote the family of all circuits of $M$ by $\mathbf{C}(M)$, i.e., $\mathbf{C}(M)=FMIN(Opp(\mathbf{I}))$.
\end{definition}

A matroid can be defined from the viewpoint of circuits, in other words, a matroid uniquely determines its circuits, and vice versa.

\begin{proposition}(Circuit axioms~\cite{Lai01Matroid})
\label{P:circuitaxiom}
Let $\mathbf{C}$ be a family of subsets of $U$.
Then there exists $M=(U, \mathbf{I})$ such that $\mathbf{C}=\mathbf{C}(M)$ if and only if $\mathbf{C}$ satisfies the following conditions:\\
(C1) $\emptyset\notin\mathbf{C}$;\\
(C2) If $C_{1}, C_{2}\in\mathbf{C}$ and $C_{1}\subseteq C_{2}$, then $C_{1}=C_{2}$;\\
(C3) If $C_{1}, C_{2}\in\mathbf{C}$, $C_{1}\neq C_{2}$ and $c\in C_{1}\cap C_{2}$, then there exists $C_{3}\in\mathbf{C}$ such that $C_{3}\subseteq C_{1}\cup C_{2}-\{c\}$.
\end{proposition}

The closure operator is one of important characteristics of matroids.
A matroid and its closure operator can uniquely determine each other.
In the following definition, we use the circuits of matroids to represent the closure operator.

\begin{definition}(Closure~\cite{Lai01Matroid})
\label{D:closure}
Let $M=(U, \mathbf{I})$ be a matroid and $X\subseteq U$.\\
\centerline{$cl_{M}(X)=X\cup\{u\in X^{c}:\exists C\in\mathbf{C}(M)$ s.t. $u\in C\subseteq X\cup\{u\}\}$}
is called the closure of $X$ with respect to $M$.
And $cl_{M}$ is the closure operator of $M$.
\end{definition}

\section{Two constructions between matroid and relation}
\label{S:twoinductionsbetweenmatroidandrelation}
In this section, we propose two constructions between a matroid and a relation.
One construction is from a relation to a matroid, and the other construction is from a matroid to a relation.

\subsection{Construction of matroid by serial and transitive relation}
\label{S:inductionofmatroidbyrelation}
In~\cite{LiuZhu12characteristicofpartition-circuitmatroid}, for an equivalence relation on a universe, we present a matroidal structure whose family of all circuits is the partition induced by the equivalence relation.
Through extending equivalence relations, in this subsection, we propose a matroidal structure which is based on a serial and transitive relation.
First, for a relation on a universe, we define a family of sets, namely, the family of all minimal neighborhoods.

\begin{definition}(The family of all minimal neighborhoods)
\label{D:thesetfamily}
Let $R$ be a relation on $U$.
We define a family of sets of $R$ as follows:\\
\centerline{$\mathbf{C}(R)=FMIN\{RN(x):x\in U\}$.}
We call $\mathbf{C}(R)$ the family of all minimal neighborhoods of $R$.
\end{definition}

In the following proposition, we will prove the family of all minimal neighborhoods of a relation on a universe satisfies the circuit axioms of matroids when the relation is serial and transitive.

\begin{proposition}
\label{P:CRsatisfiescircuitaxiom}
Let $R$ be a relation on $U$.
If $R$ is serial and transitive, then $\mathbf{C}(R)$ satisfies (C1), (C2) and (C3) in Proposition~\ref{P:circuitaxiom}.
\end{proposition}

\begin{proof}
(C1): Since $R$ is a serial relation, according to (1) of Definition~\ref{D:serialandtransitive}, then $\emptyset\notin\mathbf{C}(R)$.\\
(C2): According to Definition~\ref{D:thesetfamily} and Definition~\ref{D:symbol}, it is straightforward that $\mathbf{C}(R)$ satisfies (C2).\\
(C3): If $C_{1}, C_{2}\in\mathbf{C}(R)$ and $C_{1}\neq C_{2}$, then there exist $x_{1}, x_{2}\in U$ and $x_{1}\neq x_{2}$ such that $C_{1}=RN(x_{1}), C_{2}=RN(x_{2})$.
If $x\in C_{1}\cap C_{2}$, i.e., $x\in RN(x_{1})\cap RN(x_{2})$, then $x\in RN(x_{1})$ and $x\in RN(x_{2})$.
Since $R$ is a transitive relation, according to (2) of Definition~\ref{D:serialandtransitive}, we can obtain that $RN(x)\subseteq RN(x_{1})$ and $RN(x)\subseteq RN(x_{2})$.
According to Definition~\ref{D:thesetfamily}, we can obtain $RN(x)=RN(x_{1})=RN(x_{2})$ which is contradictory with $RN(x_{1})\neq RN(x_{2})$.
Therefore, for any $C_{1}, C_{2}\in\mathbf{C}(R)$ and $C_{1}\neq C_{2}$, $C_{1}\cap C_{2}=\emptyset$ when $R$ is transitive.
Then, it is straightforward to see $\mathbf{C}(R)$ satisfies (C3) in Proposition~\ref{P:circuitaxiom}.
\end{proof}

It is natural to ask the following question: ``when the family of all minimal neighborhoods of a relation satisfies the circuit axioms of matroids, is the relation serial and transitive?''.
In the following proposition, we will solve this issue.

\begin{proposition}
Let $R$ be a relation on $U$.
If $\mathbf{C}(R)$ satisfies (C1), (C2) and (C3) in Proposition~\ref{P:circuitaxiom}, then $R$ is serial.
\end{proposition}

\begin{proof}
According to (C1) of Proposition~\ref{P:circuitaxiom} and Definition~\ref{D:thesetfamily}, $RN(x)\neq\emptyset$ for any $x\in U$.
According to Definition~\ref{D:serialandtransitive}, $R$ is serial.
\end{proof}

When the family of all minimal neighborhoods of a relation satisfies the circuit axioms of matroids, the relation is not always a transitive relation as shown in the following example.

\begin{example}
Let $U=\{1, 2, 3\}$ and $R=\{(1, 2), (1, 3), (2, 1), (2, 3), (3, 1)\}$ a relation on $U$.
Since $RN(1)=\{2, 3\}, RN(2)=\{1, 3\}, RN(3)=\{1\}$, then $\mathbf{C}(R)=\{\{1\}, \{2, 3\}\}$.
According to Proposition~\ref{P:circuitaxiom}, $\mathbf{C}(R)$ satisfies the circuit axioms of matroids.
We see that $(1, 2)\in R, (2, 1)\in R$, but $(1, 1)\notin R, (2, 2)\notin R$, so $R$ is not transitive.
\end{example}

In this paper, we consider that the family of all minimal neighborhoods of a relation can generate a matroid when the relation is serial and transitive.

\begin{definition}
\label{D:amatroidinducedbyarelation}
Let $R$ be a serial and transitive relation on $U$.
The matroid with $\mathbf{C}(R)$ as its circuit family is denoted by $M(R)=(U, \mathbf{I}(R))$, where $\mathbf{I}(R)=Opp(Upp(\mathbf{C}(R)))$.
We say $M(R)$ is the matroid induced by $R$.
\end{definition}

The matroid induced by a serial and transitive relation can be illustrated by the following example.

\begin{example}
\label{E:example1}
Let $U=\{1, 2, 3\}$ and $R=\{(1, 1), (1, 3), (2, 1), (2, 3), (3, 3)\}$ a serial and transitive relation on $U$.
Since $RN(1)=RN(2)=\{1, 3\}, RN(3)=\{3\}$, then $\mathbf{C}(R)=\{\{3\}\}$.
Therefore, $M(R)=(U, \mathbf{I}(R))$ where $\mathbf{I}(R)=\{\emptyset, \{1\}, \{2\}, \{1, 2\}\}$.
\end{example}

Generally speaking, a matroid is defined from the viewpoint of independent sets.
In the following, we will investigate the independent sets of the matroid induced by a serial and transitive relation.

\begin{proposition}
\label{P:firsttypeofindependentsetofMR}
Let $R$ be a serial and transitive relation on $U$ and $M(R)$ the matroid induced by $R$.
Then,\\
\centerline{$\mathbf{I}(R)=\{I\subseteq U:\forall x\in U, RN(x)\nsubseteq I\}$.}
\end{proposition}

\begin{proof}
According to Definition~\ref{D:amatroidinducedbyarelation}, $\mathbf{I}(R)=Opp(Upp(\mathbf{C}(R)))$.\\
($\subseteq$): For all $X\notin Opp(Upp(\mathbf{C}(R)))$, i.e., $X\in Upp(\mathbf{C}(R))$, according to Definition~\ref{D:symbol} and Definition~\ref{D:thesetfamily}, there exists $x\in U$ such that $RN(x)\subseteq X$.
Hence $X\notin\{I\subseteq U:\forall x\in U, RN(x)\nsubseteq I\}$, i.e., $X\notin\mathbf{I}(R)$.
This proves that $\mathbf{I}(R)\subseteq Opp(Upp(\mathbf{C}(R)))$.\\
($\supseteq$): Suppose $X\notin\mathbf{I}(R)$, i.e., $X\notin\{I\subseteq U:\forall x\in U, RN(x)\nsubseteq I\}$.
Then there exists $x\in U$ such that $RN(x)\subseteq X$, in other words, $X\in Upp(\mathbf{C}(R))$, i.e., $X\notin Opp(Upp(\mathbf{C}(R)))$.
This proves that $\mathbf{I}(R)\supseteq Opp(Upp(\mathbf{C}(R)))$.
\end{proof}

To illustrate the independent sets of a matroid induced by a serial and transitive relation, the following example is given.

\begin{example}
\label{E:example2}
Let $U=\{1, 2, 3\}$ and $R=\{(1, 3), (2, 3), (3, 3)\}$ a serial and transitive relation on $U$.
Since $RN(1)=RN(2)=RN(3)=\{3\}$, then $\{3\}\nsubseteq\emptyset, \{3\}\nsubseteq \{1\}, \{3\}\nsubseteq \{2\}, \{3\}\nsubseteq \{1, 2\}$.
Therefore $M(R)=(U, \mathbf{I}(R))$ where $\mathbf{I}(R)=\{\emptyset, \{1\},$ $ \{2\}, \{1, 2\}\}$.
\end{example}

The lower and upper approximation operators are constructed through the neighborhood in generalized rough sets based on relations.
In the following proposition, we will study the independent sets of the matroid induced by a serial and transitive relation through the lower approximation operator.

\begin{proposition}
\label{P:secondtypeofindependentsetofMR}
Let $R$ be a serial and transitive relation on $U$ and $M(R)$ the matroid induced by $R$.
Then,\\
\centerline{$\mathbf{I}(R)=\{I\subseteq U:L(I)=\emptyset\}$.}
\end{proposition}

\begin{proof}
According to Definition~\ref{D:approximationoperators} and Proposition~\ref{P:firsttypeofindependentsetofMR}, it is straightforward.
\end{proof}

Because of the duality of the lower and upper approximation operators, we obtain the independent sets of the matroid induced by a serial and transitive relation through the upper approximation operator.

\begin{corollary}
Let $R$ be a serial and transitive relation on $U$ and $M(R)$ the matroid induced by $R$.
Then,\\
\centerline{$\mathbf{I}(R)=\{I\subseteq U:H(I^{c})=U\}$.}
\end{corollary}

\begin{proof}
According to (4LH) of Proposition~\ref{P:propertiesofapproximations} and Proposition~\ref{P:secondtypeofindependentsetofMR}, it is straightforward.
\end{proof}

From Example~\ref{E:example1} and Example~\ref{E:example2}, we see that two different relations on a universe generate the same matroid.
We study under what conditions two relations on a universe can generate the same matroid.
If a relation is transitive, then the reflexive closure of the relation is transitive, but the symmetric closure of the relation is not always transitive.
Therefore, we consider the relationship between the matroids induced by a serial and transitive relation and the reflexive closure of the relation.
First, we introduce the reflexive closure of a relation on a nonempty universe.

\begin{definition}(Reflexive closure of a relation~\cite{RajagopalMason92Discrete})
\label{D:thereflexiveclosure}
Let $U$ be a nonempty universe and $R$ a relation on $U$. We denote $r(R)$ as the reflexive closure of $R$, where $r(R)=R\cup\{(x, x):x\in U\}$.
\end{definition}

The relationship between the two matroids induced by a serial and transitive relation and its reflexive closure is studied in the following proposition.

\begin{proposition}
Let $R$ be a serial and transitive relation on $U$.
Then $M(R)=M(r(R))$.
\end{proposition}

\begin{proof}
According to Proposition~\ref{P:circuitaxiom}, we need only to prove $\mathbf{C}(R)=\mathbf{C}(r(R))$.
According to Definition~\ref{D:thesetfamily}, $\mathbf{C}(R)=FMIN\{RN_{R}(x):x\in U\}$ and $\mathbf{C}(r(R))=FMIN\{RN_{r(R)}(x):x\in U\}$.
According to Definition~\ref{D:thereflexiveclosure}, $r(R)=R\cup\{(x, x):x\in U\}$, then $\mathbf{C}(r(R))=FMIN\{RN_{R}(x)\cup\{x\}:x\in U\}$.\\
$(\subseteq)$: For all $RN_{R}(x)\in\mathbf{C}(R)$, on one hand, $x\in RN_{R}(x)$, then $RN_{R}(x)\in\mathbf{C}(r(R))$.
On the other hand, $x\notin RN_{R}(x)$.
Since $R$ is serial, then there exists $y\in U$ such that $y\in RN_{R}(x)$.
Since $R$ is transitive, then $RN_{R}(y)\subseteq RN_{R}(x)$.
And $RN_{R}(x)\in\mathbf{C}(R)$, then $y\in RN_{R}(y)=RN_{R}(x)$.
Therefore, $\{y\}\cup RN_{R}(y)\subset \{x\}\cup RN_{R}(x)$, i.e., $RN_{R}(y)\in\mathbf{C}(r(R))$.
In other words, $RN_{R}(x)\in\mathbf{C}(r(R))$.
Hence $\mathbf{C}(R)\subseteq\mathbf{C}(r(R))$.\\
$(\supseteq)$: For all $RN_{r(R)}(x)\in\mathbf{C}(r(R))$, i.e., $RN_{R}(x)\cup\{x\}\in\mathbf{C}(r(R))$.
If $x\in RN_{R}(x)$, then $RN_{r(R)}(x)\in\mathbf{C}(R)$.
If $x\notin RN_{R}(x)$, i.e., $RN_{R}(x)\subset RN_{r(R)}(x)$, then $RN_{R}(x)\in\mathbf{C}(R)$.
If $RN_{R}(x)\notin\mathbf{C}(R)$, then there exists $y\in U$ such that $RN_{R}(y)\subset RN_{R}(x)$ and $RN_{R}(y)\in\mathbf{C}(R)$.
Since $R$ is serial, then there exists $z\in U$ such that $z\in RN_{R}(y)$.
And $R$ is transitive, then $RN_{R}(z)\subseteq RN_{R}(y)$.
Therefore $z\in RN_{R}(z)=RN_{R}(y)$, i.e., $RN_{R}(z)\cup\{z\}=RN_{R}(y)$.
Hence $RN_{R}(z)\cup\{z\}\subset RN_{R}(x)\subset RN_{R}(x)\cup\{x\}$ which is contradictory with $RN_{R}(x)\cup\{x\}\in\mathbf{C}(r(R))$.
Since $R$ is serial, then there exists $y\in U$ such that $y\in RN_{R}(x)$.
Since $R$ is transitive, then $RN_{R}(y)\subseteq RN_{R}(x)$.
And $RN_{R}(x)\in\mathbf{C}(R)$, then $y\in RN_{R}(y)=RN_{R}(x)$, i.e., $\{y\}\cup RN_{R}(y)\subset \{x\}\cup RN_{R}(x)$ which is contradictory with $RN_{R}(x)\cup\{x\}\in\mathbf{C}(r(R))$.
Hence $RN_{r(R)}(x)\in\mathbf{C}(R)$, i.e., $\mathbf{C}(r(R))\subseteq\mathbf{C}(R)$.
\end{proof}

In fact, for a universe, any different relations generate the same matroid if and only if the families of all minimal neighborhoods are equal to each other according to Proposition~\ref{P:circuitaxiom}.

It is known the upper approximation operator induced by a reflexive and transitive relation is exactly the closure operator of a topology~\cite{Kortelainen94onrelationship,QinYangPei08generalizedroughsets,YangXu11topologicalproperties}.
We see a matroid has its own closure operator.
In this paper, we will not discuss the relationship between the closure operator of a topology and one of a matroid.
We will study the connection between the upper approximation operator of a rough set and the closure operator of a matroid.
In the following, we investigate the relationship between the upper approximation operator of a relation and the closure operator of the matroid induced by the relation.

In a matroid, the closure of any subset contains the subset itself.
When a relation on a universe is reflexive,  the upper approximation of any subset contains the subset itself.
If a relation is reflexive, then it is serial.
According to Definition~\ref{D:thesetfamily}, a reflexive and transitive relation can induce a matroid.
We will study the relationship between the upper approximation operator of a reflexive and transitive relation and the closure operator of the matroid induced by the relation.
A counterexample is given in the following.

\begin{example}
Let $U=\{1, 2, 3\}$ and $R=\{(1, 1), (1, 3), (2, 2), (2, 3), (3, 3)\}$ a reflexive and transitive relation on $U$.
Since $RN(1)=\{1, 3\}, RN(2)=\{2, 3\}, RN(3)=\{3\}$, then $\mathbf{C}(R)=\{\{3\}\}$.
Therefore, $cl_{M(R)}(\emptyset)=\{3\}, cl_{M(R)}(\{3\})=\{3\}$.
Since $H_{R}(\emptyset)=\emptyset, H_{R}(\{3\})=\{1, 2, 3\}$, then $H_{R}(\emptyset)\subseteq cl_{M(R)}(\emptyset)$ and $H_{R}(\{3\})\supseteq cl_{M(R)}(\{3\})$.
\end{example}

From the above example, we see that the closure operator of the matroid induced by a reflexive and transitive relation dose not correspond to the upper approximation operator of the relation.
A condition, when the closure operator contains the upper approximation operator, is studied in the following proposition.
First, we present a remark.

\begin{remark}
Suppose $R$ is an equivalence relation on $U$.
According to Definition~\ref{D:approximationoperators} and Proposition~\ref{P:reflexiverelation}, $H_{R}(X)=X\cup\{u\in X^{c}:\exists x\in X$ s.t. $x\in RN_{R}(u)\}$.
According to Definition~\ref{D:thesetfamily}, $\mathbf{C}(R)=\{RN_{R}(x):x\in U\}$.
\end{remark}

\begin{proposition}
\label{P:closureofinducedmatroidcontainsupper}
Let $R$ be an equivalence relation on $U$ and $M(R)$ the matroid induced by $R$.
If for all $C\in\mathbf{C}(R), |C|\leq 2$, then $H_{R}(X)\subseteq cl_{M(R)}(X)$ for all $X\subseteq U$.
\end{proposition}

\begin{proof}
According to Definition~\ref{D:closure}, we need only to prove that $\{u\in X^{c}:\exists x\in X$ s.t. $x\in RN_{R}(u)\}\subseteq\{u\in X^{c}:RN_{R}(u)\subseteq X\cup\{u\}\}$.
Since $R$ is an equivalence relation, and for all $C\in\mathbf{C}(R), |C|\leq 2$, then $x\in RN_{R}(x)$ and $|RN_{R}(x)|\leq 2$ for all $x\in U$.
When $|RN_{R}(x)|=1$, i.e., $RN_{R}(x)=\{x\}$, then $\{u\in X^{c}:\exists x\in X$ s.t. $x\in RN_{R}(u)\}=\emptyset$, therefore $\{u\in X^{c}:\exists x\in X$ s.t. $x\in RN_{R}(u)\}\subseteq\{u\in X^{c}:RN_{R}(u)\subseteq X\cup\{u\}\}$.
When $|RN_{R}(x)|=2$, for all $u\in\{u\in X^{c}:\exists x\in X$ s.t. $x\in RN_{R}(u)\}, RN_{R}(u)=\{x, u\}$, then $RN_{R}(u)\subseteq X\cup\{u\}$, i.e., $\{u\in X^{c}:\exists x\in X$ s.t. $x\in RN_{R}(u)\}\subseteq\{u\in X^{c}:RN_{R}(u)\subseteq X\cup\{u\}\}$.
To sum up, this completes the proof.
\end{proof}

In fact, for an equivalence relation on a universe and any subset of the universe, its closure with respect to the induced matroid can be expressed by the union of its upper approximation with respect to the relation and the family of some elements whose neighborhood is equal to itself.

\begin{proposition}
Let $R$ be an equivalence relation on $U$ and $M(R)$ the matroid induced by $R$.
If for all $C\in\mathbf{C}(R), |C|\leq 2$, then $cl_{M(R)}(X)=H_{R}(X)\cup\{u\in X^{c}: RN_{R}(u)=\{u\}\}$ for all $X\subseteq U$.
\end{proposition}

\begin{proof}
We need only to prove $\{u\in X^{c}:RN_{R}(u)\subseteq X\cup\{u\}\}=\{u\in X^{c}:\exists x\in X$ s.t. $x\in RN_{R}(u)\}\cup\{u\in X^{c}: RN_{R}(u)=\{u\}\}$.
Since $R$ is equivalence relation and $|C|\leq 2$ for all $C\in\mathbf{C}(R)$, then $|RN_{R}(x)|\leq 2$ for all $x\in U$.
Therefore, it is straightforward to obtain $\{u\in X^{c}:RN_{R}(u)\subseteq X\cup\{u\}\}=\{u\in X^{c}:\exists x\in X$ s.t. $x\in RN_{R}(u)\}\cup\{u\in X^{c}: RN_{R}(u)=\{u\}\}$, i.e., $cl_{M(R)}(X)=H_{R}(X)\cup\{u\in X^{c}: RN_{R}(u)=\{u\}\}$.
\end{proof}

Similarly, can the upper approximation operator of an equivalence relation contain the closure operator of the matroid induced by the equivalence relation when the cardinality of any circuit of the matroid is equal or greater than 2?

\begin{proposition}
\label{P:uppercontainsclosureofinducedmatroid}
Let $R$ be an equivalence relation on $U$ and $M(R)$ the matroid induced by $R$.
If for all $C\in\mathbf{C}(R), |C|\geq 2$, then $cl_{M(R)}(X)\subseteq H_{R}(X)$ for all $X\subseteq U$.
\end{proposition}

\begin{proof}
According to Definition~\ref{D:closure}, we need only to prove $\{u\in X^{c}:\exists C\in\mathbf{C}(R)$ s.t. $u\in C\subseteq X\cup\{u\}\}\subseteq\{u\in X^{c}:\exists x\in X$ s.t. $x\in RN_{R}(u)\}$.
For all $u\in\{u\in X^{c}:\exists C\in\mathbf{C}(R)$ s.t. $u\in C\subseteq X\cup\{u\}\}$, then $RN_{R}(u)\subseteq X\cup\{u\}$.
And $|RN_{R}(u)|\geq 2$, then there exists at least an element $x\in X$ such that $x\in RN_{R}(u)$, i.e., $u\in\{u\in X^{c}:\exists x\in X$ s.t. $x\in RN_{R}(u)\}$.
To sum up, this completes the proof.
\end{proof}

A sufficient and necessary condition, when the upper approximation operator of an equivalence relation is equal to the closure operator of the matroid induced by the equivalence relation, is investigated in the following theorem.
First, we introduce a special matroid called 2-circuit matroid.

\begin{definition}(2-circuit matroid~\cite{WangZhuZhuMin12matroidalstructure})
\label{D:2-circuitmatroid}
Let $M=(U, \mathbf{I})$ be a matroid.
If for all $C\in\mathbf{C}(M), |C|=2$, then we say $M$ is a 2-circuit matroid.
\end{definition}

\begin{theorem}
Let $R$ be an equivalence relation on $U$ and $M(R)$ the matroid induced by $R$.
$M(R)$ is a 2-circuit matroid iff $H_{R}(X)=cl_{M(R)}(X)$ for all $X\subseteq U$.
\end{theorem}

\begin{proof}
According to Definition~\ref{D:closure},we need only to prove for all $x\in U$, $|RN_{R}(x)|=2$ iff $\{u\in X^{c}:\exists x\in X$ s.t. $x\in RN_{R}(u)\}=\{u\in X^{c}:RN_{R}(u)\subseteq X\cup\{u\}\}$ for all $X\subseteq U$.\\
$(\Rightarrow)$: According to Proposition~\ref{P:closureofinducedmatroidcontainsupper} and Proposition~\ref{P:uppercontainsclosureofinducedmatroid}, it is straightforward.\\
$(\Leftarrow)$: We prove this by reductio.\\
On one hand, suppose there exists $x\in U$ such that $|RN_{R}(x)|=1$.
Suppose $X=\emptyset$.
Then $x\in\{u\in X^{c}:RN_{R}(u)\subseteq X\cup\{u\}\}$.
According to (1H) of Proposition~\ref{P:propertiesofapproximations}, $H_{R}(\emptyset)=\emptyset$.
Therefore $H_{R}(\emptyset)\subset cl_{M(R)}(\emptyset)$ which is contradictory with $H_{R}(X)=cl_{M(R)}(X)$ for all $X\subseteq U$.\\
On the other hand, suppose there exists $y\in U$ such that $|RN_{R}(y)|\geq 3$.
Suppose $X\cap RN_{R}(y)=\{x\}$, then $y\in\{u\in X^{c}:\exists x\in X$ s.t. $x\in RN_{R}(u)\}$.
Since $y\in RN_{R}(y)\nsubseteq X$, then $RN_{R}(y)\nsubseteq X\cup\{y\}$, which is contradictory with $\{u\in X^{c}:\exists x\in X$ s.t. $x\in RN_{R}(u)\}=\{u\in X^{c}:RN_{R}(u)\subseteq X\cup\{u\}\}$ for all $X\subseteq U$.
Therefore, for all $x\in U$, $|RN_{R}(x)|=2$.
\end{proof}

\subsection{Construction of equivalence relation by matroid}
\label{S:inductionofrelationbymatroid}
In order to further study the matroidal structure of the rough set based on a serial and transitive relation, we consider the inverse of the construction in Subsection~\ref{S:inductionofmatroidbyrelation}: inducing a relation by a matroid.
Firstly, through the connectedness in a matroid, a relation can be obtained.

\begin{definition}(\cite{Oxley93Matroid})
\label{D:relationinducedbymatroid}
Let $M=(U, \mathbf{I})$ be a matroid.
We define a relation $R(M)$ on $U$ as follows: for all $x, y\in U$,\\
\centerline{$(x, y)\in R(M)\Leftrightarrow x=y$ or $\exists C\in\mathbf{C}(M)$ s.t. $\{x, y\}\subseteq C$.}
We say $R(M)$ is induced by $M$.
\end{definition}

The following example is to illustrate the construction of a relation from a matroid.

\begin{example}
Let $M=(U, \mathbf{I})$ be a matroid, where $U=\{1, 2, 3\}$ and $\mathbf{I}=\{\emptyset, \{1\}\}$.
Since $\mathbf{C}(M)=\{\{2\}, \{3\}\}$, then according to Definition~\ref{D:relationinducedbymatroid}, $R(M)=\{(1, 1), (2, 2), (3, 3)\}$.
\end{example}

In fact, according to Definition~\ref{D:relationinducedbymatroid}, the relation induced by a matroid is an equivalence relation.

\begin{proposition}(\cite{Oxley93Matroid})
\label{P:theinducedrelationisequivalencerelation}
Let $M=(U, \mathbf{I})$ be a matroid.
Then $R(M)$ is an equivalence relation on $U$.
\end{proposition}

The following example is presented to illustrate that different matroids generate the same relation.

\begin{example}
Let $M_{1}=(U, \mathbf{I}_{1}), M_{2}=(U, \mathbf{I}_{2})$ be two matroids where $U=\{1, 2, 3\}, \mathbf{I}_{1}=\{\emptyset, \{1\}, \{2\}\}$ and $\mathbf{I}_{2}=\{\emptyset, \{1\}, \{2\}, \{3\}, \{1, 3\}, \{2, 3\}\}$.
Since $\mathbf{C}(M_{1})=\{\{1, 2\}, \{3\}\},$ $ \mathbf{C}(M_{2})=\{\{1, 2\}\}$, according to Definition~\ref{D:relationinducedbymatroid}, then $R(M_{1})=R(M_{2})=\{(1, 1), (1, 2),$ $ (2, 1), (2,$ $ 2), (3, 3)\}$.
\end{example}

Similarly, we will study the relationship between the closure operator of a matroid and the upper approximation operator of the equivalence relation induced by the matroid in the following proposition.

\begin{proposition}
\label{P:closureofmatroidcontainedbyupperoftheinducedrelation}
Let $M=(U, \mathbf{I})$ be a matroid and $R(M)$ the relation induced by $M$.
If $cl_{M}(\emptyset)=\emptyset$, then $cl_{M}(X)\subseteq H_{R(M)}(X)$ for all $X\subseteq U$.
\end{proposition}

\begin{proof}
Since $R(M)$ is an equivalence relation, according to Definition~\ref{D:closure}, we need only to prove $\{u\in X^{c}:\exists C\in\mathbf{C}(M)$ s.t. $u\in C\subseteq X\cup\{u\}\}\subseteq\{u\in X^{c}:\exists x\in X$ s.t. $x\in RN_{R(M)}(u)\}$.
Since $cl_{M}(\emptyset)=\emptyset$, then $\{x\}\notin\mathbf{C}(M)$ for all $x\in U$.
For all $u\in\{u\in X^{c}:\exists C\in\mathbf{C}(M)$ s.t. $u\in C\subseteq X\cup\{u\}\}$, there exists at least one element $x\in X$ such that $\{x, u\}\subseteq C\subseteq X\cup\{u\}$.
According to Definition~\ref{D:relationinducedbymatroid}, $(u, x)\in R(M)$, i.e., $x\in RN_{R(M)}(u)$.
Therefore $u\in\{u\in X^{c}:\exists x\in X$ s.t. $x\in RN_{R(M)}(u)\}$.
To sum up, this completes the proof.
\end{proof}

The above proposition can be illustrated by the following example.

\begin{example}
Let $M=(U, \mathbf{I})$ be a matroid where $U=\{1, 2, 3\}$ and $\mathbf{I}=\{\emptyset, \{1\}, \{2\}, $ $\{3\}, \{1, 2\}, \{1, 3\}, \{2, 3\}\}$.
Since $\mathbf{C}(M)=\{\{1, 2, 3\}\}$, then $R(M)=U\times U$.
Therefore $cl_{M}(\emptyset)=\emptyset, cl_{M}(\{1\})=\{1\}, cl_{M}(\{2\})=\{2\}, cl_{M}(\{3\})=\{3\}, cl_{M}(\{1, 2\})=cl_{M}(\{1, 3\})=cl_{M}(\{2, 3\})=cl_{M}(\{1, 2, 3\})=\{1, 2, 3\}$ and $H_{R(M)}(\emptyset)=\emptyset, H_{R(M)}($ $\{1\})=H_{R(M)}(\{2\})=H_{R(M)}(\{3\})=H_{R(M)}(\{1, 2\})=H_{R(M)}(\{1, 3\})=H_{R(M)}($ $\{2, 3\})=H_{R(M)}(\{1, 2, 3\})=\{1, 2, 3\}$.
Hence for all $X\subseteq U$, $cl_{M}(X)\subseteq H_{R(M)}(X)$.
\end{example}

According to Proposition~\ref{P:closureofmatroidcontainedbyupperoftheinducedrelation}, the closure operator of a matroid is contained in the upper approximation operator of the relation induced by the matroid, when the closure of empty set is equal to empty set.
We consider an issue that when the closure of empty set is not equal to empty set, can the closure operator contain the upper approximation operator?
A counterexample is given in the following.

\begin{example}
Let $U=\{1, 2, 3, 4\}$ and $M$ a matroid on $U$, where $\mathbf{C}(M)=\{\{1, 2, 3\}, $ $\{4\}\}$.
Since $cl_{M}(\emptyset)=\{4\}$, then $cl_{M}(\{1\})=\{1, 4\}$.
Since $U/R(M)=\mathbf{C}(M)$, then $H_{R(M)}(\{1\})=\{1, 2, 3\}$.
Therefore, $H_{R(M)}(\{1\})\nsubseteq cl_{M}(\{1\})$.
\end{example}

Under what condition the closure operator contains the upper approximation operator?
In the following proposition, we study this issue.

\begin{proposition}
\label{P:closuecontainsupper}
Let $M=(U, \mathbf{I})$ be a matroid and $R(M)$ the relation induced by $M$.
If for all $C\in\mathbf{C}(M), |C|\leq 2$, then $H_{R(M)}(X)\subseteq cl_{M}(X)$ for all $X\subseteq U$.
\end{proposition}

\begin{proof}
According to Proposition~\ref{P:theinducedrelationisequivalencerelation}, Definition~\ref{D:approximationoperators} and Definition~\ref{D:closure}, we need only to prove $\{u\in X^{c}:\exists x\in X$ s.t. $x\in RN_{R(M)}(u)\}\subseteq\{u\in X^{c}:\exists C\in\mathbf{C}(M)$ s.t. $u\in C\subseteq X\cup\{u\}\}$.
For all $u\in\{u\in X^{c}:\exists x\in X$ s.t. $x\in RN_{R(M)}(u)\}$, then $(u, x)\in R(M)$.
According to Definition~\ref{D:relationinducedbymatroid}, there exists $C\in\mathbf{C}(M)$ such that $\{x, u\}\subseteq C$.
Since for all $C\in\mathbf{C}(M), |C|\leq 2$, then $C=\{x, u\}$, i.e., $u\in C\subseteq X\cup\{u\}$.
Therefore, $u\in\{u\in X^{c}:\exists C\in\mathbf{C}(M)$ s.t. $u\in C\subseteq X\cup\{u\}\}$.
To sum up, this completes the proof.
\end{proof}

In the following theorem, we investigate a sufficient and necessary condition when the closure operator of a matroid is equal to the upper approximation operator of the relation induced by the matroid.

\begin{theorem}
Let $M=(U, \mathbf{I})$ be a matroid and $R(M)$ the relation induced by $M$.
For all $X\subseteq U$, $H_{R(M)}(X)=cl_{M}(X)$ iff $\mathbf{C}(M)=\emptyset$ or $M$ is a 2-circuit matroid.
\end{theorem}

\begin{proof}
(1) Since $\mathbf{C}(M)=\emptyset$, according to Definition~\ref{D:relationinducedbymatroid}, $R(M)=\{(x, x):x\in U\}$.
According to Definition~\ref{D:closure}, for all $X\subseteq U$,  $cl_{M}(X)=X$.
And according to Definition~\ref{D:approximationoperators}, for all $X\subseteq U$, $H_{R(M)}(X)=X$.
Therefore, $cl_{M}(X)=H_{R(M)}(X)=X$ for all $X\subseteq U$.
Similarly, if $cl_{M}(X)=H_{R(M)}(X)=X$, then $\mathbf{C}(M)=\emptyset$.\\
(2) According to Proposition~\ref{P:theinducedrelationisequivalencerelation}, Definition~\ref{D:approximationoperators} and Definition~\ref{D:closure}, we need only to prove $\{u\in X^{c}:\exists x\in X$ s.t. $x\in RN_{R(M)}(u)\}=\{u\in X^{c}:\exists C\in\mathbf{C}(M)$ s.t. $u\in C\subseteq X\cup\{u\}\}$ iff $M$ is a 2-circuit matroid.\\
$(\Leftarrow)$: Since $M$ is a 2-circuit matroid, then for all $C\in\mathbf{C}(M), |C|=2$.
$\{u\in X^{c}:\exists x\in X$ s.t. $x\in RN_{R(M)}(u)\}=\{u\in X^{c}:\exists x\in X$ s.t. $(x, u)\in R(M)\}=\{u\in X^{c}:\exists x\in X$ s.t. $\exists C\in\mathbf{C}(M), \{x, u\}\subseteq C\}=\{u\in X^{c}:\exists C\in\mathbf{C}(M),$ s.t. $\{x, u\}= C\subseteq X\cup\{u\}\}=\{u\in X^{c}:\exists C\in\mathbf{C}(M)$ s.t. $u\in C\subseteq X\cup\{u\}\}$.\\
$(\Rightarrow)$: According to Proposition~\ref{P:closureofmatroidcontainedbyupperoftheinducedrelation} and Proposition~\ref{P:closuecontainsupper}, it is straightforward.
\end{proof}

\section{Relationships between the two constructions}
\label{S:relationshipsbetweenthetwoinductions}
In this section, we study the relationships between the two constructions in Section~\ref{S:twoinductionsbetweenmatroidandrelation}.
The first construction takes a relation and yields a matroid, and the second construction takes a matroid and then yields a relation.
Firstly, given a matroid, it can generate an equivalence relation, and the equivalence relation can generate a matroid, then the connection between the original matroid and the induced matroid is built.
In order to study the connection, we introduce the definition of the finer family of subsets on a universe.

\begin{definition}
\label{D:thefiner}
Let $\mathbf{F}_{1}, \mathbf{F}_{2}$ be two families of subsets on $U$.
If for all $F_{1}\in\mathbf{F}_{1}$, there exists $F_{2}\in\mathbf{F}_{2}$ such that $F_{1}\subseteq F_{2}$, then we say $\mathbf{F}_{1}$ is finer than $\mathbf{F}_{2}$, and denote it as $\mathbf{F}_{1}\leq\mathbf{F}_{2}$.
\end{definition}

In the following proposition, we will represent the relationship between the circuits of a matroid and the circuits of the matroid induced by the equivalence relation which is generated by the original matroid.

\begin{proposition}
\label{P:thefinercircuit}
Let $M=(U, \mathbf{I})$ be a matroid.
Then $\mathbf{C}(M)\leq\mathbf{C}(M(R(M)))$.
\end{proposition}

\begin{proof}
For all $C\in\mathbf{C}(M)$, suppose $x, y\in C$.
According to Definition~\ref{D:relationinducedbymatroid}, we can obtain $(x, y)\in R(M)$.
According to Proposition~\ref{P:theinducedrelationisequivalencerelation}, $R(M)$ is an equivalence relation on $U$, then $x, y\in RN_{R(M)}(x)$.
According to Definition~\ref{D:thesetfamily} and Definition~\ref{D:amatroidinducedbyarelation}, $\mathbf{C}(M(R(M)))=FMIN\{RN_{R(M)}(z):z\in U\}$.
Since $FMIN\{RN_{R(M)}(z):z\in U\}=U/R(M)$ and $RN_{R(M)}(x)\in U/R(M)$, then there exists $RN_{R(M)}(x)\in\mathbf{C}(M(R(M)))$ such that $C\subseteq RN_{R(M)}(x)$.
According to Definition~\ref{D:thefiner}, $\mathbf{C}(M)\leq\mathbf{C}(M(R(M)))$.
\end{proof}

In order to further comprehend Proposition~\ref{P:thefinercircuit}, the following example is given.

\begin{example}
Let $M=(U, \mathbf{I})$ a matroid, where $U=\{1, 2, 3\}$ and $\mathbf{I}=\{\emptyset, \{1\}, \{2\}, \{3\}\}$.
Since $\mathbf{C}(M)=\{\{1, 2\}, \{1, 3\}, \{2, 3\}\}$, according to Definition~\ref{D:relationinducedbymatroid}, $R(M)=\{(1, 1),$ $ (1, 2), (1, 3), (2, 2), (2, 1), (2, 3), (3, 3), (3, 1), (3, 2)\}$.
According to Definition~\ref{D:thesetfamily} and Definition~\ref{D:amatroidinducedbyarelation}, $\mathbf{C}(M(R(M)))=\{\{1, 2, 3\}\}$.
Therefore $\mathbf{C}(M)\leq\mathbf{C}(M(R(M)))$.
\end{example}

A matroid can induce an equivalence relation, and the equivalence relation can generate a matroid, then a sufficient and necessary condition when the original matroid is equal to the induced matroid is studied in the following theorem.

\begin{theorem}
Let $M=(U, \mathbf{I})$ be a matroid.
Then, $M(R(M))=M$ if and only if $\mathbf{C}(M)$ is a partition on $U$.
\end{theorem}

\begin{proof}
According to Proposition~\ref{D:circuit}, we need only to prove $\mathbf{C}(M)=\mathbf{C}(M(R(M)))$ iff $\mathbf{C}(M)$ is a partition on $U$.\\
$(\Rightarrow)$: According to Proposition~\ref{P:theinducedrelationisequivalencerelation}, $R(M)$ is an equivalence relation on $U$.
According to Definition~\ref{D:thesetfamily} and Definition~\ref{D:amatroidinducedbyarelation}, $\mathbf{C}(M(R(M)))=U/R(M)$.
Since $\mathbf{C}(M)=\mathbf{C}(M(R(M)))$, then $\mathbf{C}(M)$ is a partition on $U$.\\
$(\Leftarrow)$: If $\mathbf{C}(M)$ is a partition on $U$, according to Definition~\ref{D:relationinducedbymatroid}, then $U/R(M)=\mathbf{C}(M)$.
According to Definition~\ref{D:thesetfamily} and Definition~\ref{D:amatroidinducedbyarelation}, then $\mathbf{C}(M(R(M)))=U/R(M)$.
Therefore $\mathbf{C}(M)=\mathbf{C}(M(R(M)))$.
\end{proof}

Similarly, a serial and transitive relation can generate a matroid, and the matroid can generate an equivalence relation, then the relationship between the original relation and the induced equivalence relation is studied as follows.
First, we present a lemma about the transitivity of a relation.

\begin{lemma}
\label{L:transitive_therelationshipbetweentheminimalneighbors}
Let $R$ be a transitive relation on $U$.
For all $RN(x), RN(y)\in FMIN\{RN(z):z\in U\}$, if $RN(x)\neq RN(y)$, then $RN(x)\cap RN(y)=\emptyset$.
\end{lemma}

\begin{proof}
Suppose $RN(x)\cap RN(y)\neq\emptyset$, then there exists $z\in U$ such that $z\in RN(x)\cap RN(y)$, i.e., $z\in RN(x), z\in RN(y)$.
According to Definition~\ref{D:serialandtransitive}, $RN(z)\subseteq RN(x),$ $ RN(z)\subseteq RN(y)$.
Since $RN(x)\neq RN(y)$, then $RN(z)\subset RN(x), RN(z)\subset RN(y)$, which is contradictory with $RN(x), RN(y)\in FMIN\{RN(z):z\in U\}$.
Therefore, $RN(x)\cap RN(y)=\emptyset$.
\end{proof}

If a relation is reflexive, then it is also serial.
Therefore, a reflexive and transitive relation can generate a matroid according to Definition~\ref{D:amatroidinducedbyarelation}.

\begin{proposition}
Let $R$ be a reflexive and transitive relation on $U$.
Then $R(M(R))\subseteq R$.
\end{proposition}

\begin{proof}
According to Definition~\ref{D:relationinducedbymatroid}, Definition~\ref{D:thesetfamily} and Definition~\ref{D:amatroidinducedbyarelation}, we can obtain that for all $(x, y)\in R(M(R))$, if $x\neq y$, then there exists $z\in U$ such that $\{x, y\}\subseteq RN(z)\in FMIN\{RN_{R}(x):x\in U\}$, i.e., $x\in RN_{R}(z), y\in RN_{R}(z)$.
Since $R$ is reflexive and transitive, then $x\in RN_{R}(x)\subseteq RN_{R}(z)$, i.e., $x\in RN_{R}(x)\cap RN_{R}(z)$.
Since $RN_{R}(z)\in FMIN\{RN_{R}(x):x\in U\}$, according to Lemma~\ref{L:transitive_therelationshipbetweentheminimalneighbors}, $RN_{R}(x)=RN_{R}(z)$, then $z\in RN_{R}(x)$, i.e., $(x, z)\in R$.
Since $y\in RN_{R}(z)$, i.e., $(z, y)\in R$, and $R$ is transitive, then $(x, y)\in R$.
Since $R$ is reflexive, then $(x, x)\in R$ for all $x\in U$.
Hence $R(M(R))\subseteq R$.
\end{proof}

The above proposition can be illustrated by the following example.

\begin{example}
Let $U=\{1, 2, 3, 4\}$ and $R=\{(1, 1), (1, 2), (1, 3), (2, 2), (2, 3), (3, 3), $ $(3, 2), (4, 4)\}$ a reflexive and transitive relation on $U$.
Since $RN_{R}(1)=\{1, 2, 3\}, $ $RN_{R}(2)=RN_{R}(3)=\{2, 3\}, RN_{R}(4)=\{4\}$, then $\mathbf{C}(M(R))=FMIN\{RN_{R}(x):x\in U\}=\{\{2, 3\}, \{4\}\}$.
Therefore, $R(M(R))=\{(1, 1), (2, 2), (2, 3), (3, 3), (3, 2),$ $ (4, 4)\}$.
We can see that $R(M(R))\subseteq R$.
\end{example}

A sufficient and necessary condition, when the relation is equal to the induced equivalence relation, is investigated in the following theorem.

\begin{theorem}
Let $R$ be a serial and transitive relation on $U$.
Then, $R(M(R))=R$ if and only if $R$ is an equivalence relation.
\end{theorem}

\begin{proof}
$(\Rightarrow)$: According to Proposition~\ref{P:theinducedrelationisequivalencerelation}, $R(M(R))$ is an equivalence relation on $U$.
Since $R(M(R))=R$, then $R$ is an equivalence relation.\\
$(\Leftarrow)$: $R$ is an equivalence relation, then $\mathbf{C}(M(R))=U/R$.
According to Definition~\ref{D:relationinducedbymatroid} and Proposition~\ref{P:theinducedrelationisequivalencerelation}, it is straightforward to prove that $R(M(R))=R$.
\end{proof}


\section{Conclusions}
\label{S:conclusions}
In order to broaden the theoretical and application fields of rough sets and matroids, their connections with other theories have been built.
In this paper, we connected matroids and generalized rough sets based on relations.
For a serial and transitive relation on a universe, we proposed a matroidal structure through the neighborhood of the relation.
First, we defined the family of all minimal neighborhoods of a relation on a universe, and proved it to satisfy the circuit axioms of matroids when the relation was serial and transitive.
The independent sets of the matroid were studied, and the connections between the upper approximation operator of the relation and the closure operator of the matroid were investigated.
In order to study the matroidal structure of the rough set based on a serial and transitive, we investigated the inverse of the above construction: inducing a relation by a matroid.
Through the connectedness in a matroid, a relation was obtained and proved to be an equivalence relation.
And the closure operator of the matroid was equal to the upper approximation operator of the induced equivalence relation if and only if the matroid was a 2-circuit matroid.
Second, the relationships between the above two constructions were investigated.
For a matroid on a universe, it induced an equivalence relation, and the equivalence relation generated a matroid, then the original matroid was equal to the induced matroid if and only if the family of circuits of the original matroid was a partition on the universe.
For a serial and transitive relation on a universe, it generated a matroid, and the matroid induced an equivalence relation, then the original relation was equal to the induced equivalence relation if and only if the original relation was an equivalence relation.

\section*{Acknowledgments}
This work is supported in part by the National Natural Science Foundation of China under Grant No. 61170128, the Natural Science Foundation of Fujian Province, China, under Grant Nos. 2011J01374 and 2012J01294, the Science and Technology Key Project of Fujian Province, China, under Grant No. 2012H0043.


\end{document}